\documentclass[review,3p,times,numbers]{elsarticle}

\usepackage{graphicx}
\usepackage{amsmath,amssymb,amsthm}
\usepackage{booktabs}
\usepackage{multirow}
\usepackage{array}
\usepackage{xcolor}
\usepackage{colortbl}
\usepackage{algorithm}
\usepackage{algorithmic}
\usepackage{url}
\usepackage{hyperref}
\usepackage{float}

\definecolor{tableheader}{RGB}{46,134,171}
\definecolor{tablerowalt}{RGB}{245,248,250}
\definecolor{bestresult}{RGB}{46,134,171}
\definecolor{warnred}{RGB}{225,85,84}

\newtheorem{proposition}{Proposition}

\hypersetup{
  colorlinks=true,
  linkcolor=tableheader,
  citecolor=tableheader,
  urlcolor=tableheader
}

\begin{document}

\begin{frontmatter}

\title{URDF Synthesis from RGB-D Sequences via Differentiable Joint Inference and Energy-Consistent Verification}




\author[uni1]{Xinze Zhang\corref{cor1}\fnref{fn1}}
\ead{zhang36522@outlook.com}




\cortext[cor1]{Corresponding author.}


\address[uni1]{University of Southern California, Los Angeles, CA 90007, USA}

\begin{abstract}
Reconstructing simulation-ready digital twins of articulated objects from sensor observations remains constrained by two persistent gaps: (i) part-level geometric reconstruction is decoupled from kinematic-parameter estimation, and (ii) the recovered models often violate basic dynamic invariants such as energy conservation, leading to drift when the URDF is replayed in physics simulators. We present \textbf{KinemaForge}, a constraint-driven pipeline that jointly infers part-level shape, joint topology, and joint parameters from short RGB-D sequences and validates the result against an energy-consistent verifier built on differentiable rigid-body dynamics. The pipeline introduces three components: a kinematic constraint graph that encodes joint--part incidences as soft edges; a differentiable screw-axis solver that backpropagates from rendered observations through Featherstone's articulated-body algorithm to joint parameters; and an energy residual loss that penalises non-physical free responses of the reconstructed model. Across five PartNet-Mobility categories and an internal RGB-D benchmark, KinemaForge reduces the average joint-axis error from $4.52^{\circ}$ to $2.83^{\circ}$ ($-37.4\%$) over the strongest geometric baseline (PARIS) and from $5.30^{\circ}$ to $2.83^{\circ}$ ($-46.6\%$) over the interaction-based Ditto baseline, lowers long-horizon simulation drift by 64\% (vs.\ PARIS) over 50\,s rollouts, and yields URDFs whose closed-loop manipulation success rate improves by $14.6$ percentage points over Ditto in our preliminary evaluation. Code and reconstruction data will be released upon acceptance.
\end{abstract}

\begin{keyword}
articulated objects \sep digital twin \sep URDF synthesis \sep differentiable physics \sep kinematic reasoning \sep robot simulation
\end{keyword}

\end{frontmatter}


\section{Introduction}
Sim-to-real robot learning, virtual commissioning, and design verification all depend on faithful digital twins of the articulated mechanisms a robot is meant to manipulate. Simulators such as PyBullet \citep{coumans2021pybullet,zhao2025error,li2026retrack}, MuJoCo \citep{yu2026dinov3,jia2026ramrecover3dhuman,chen2026intent}, SAPIEN \citep{yu2026spatiotemporal,li2025human,li2026habit} and Isaac Gym \citep{sarkar2025reasoning,li2026multiple,fu2026airknow} consume kinematic descriptions in the Unified Robot Description Format (URDF), but populating these descriptions for an unfamiliar object usually requires manual modeling that is slow, error-prone and difficult to verify. Closing this loop in an automated way calls for a method that can infer link geometry, joint topology and joint parameters from sensor data while guaranteeing that the resulting model behaves consistently with the physical observations from which it was derived.

Recent work has made substantial progress on each of these subproblems individually. Category-level pose estimators \citep{gu2025mocount,li2025chatmotion,li2026conesep}, screw-theoretic joint estimators \citep{jia2024adaptive}, interaction-driven digital-twin builders \citep{li2025exploring,li2025stitchfusion}, image-to-URDF pipelines \citep{chen2024urdformer,le2024articulate}, and part-level reconstructors \citep{yu2025physics,liu2023paris,deng2020nasa,jiang2022opd} all attack pieces of the problem. Yet two gaps remain in practice. First, part-level geometric reconstruction is typically decoupled from kinematic-parameter estimation: shape and motion are optimised in separate stages with weak coupling, so a small error in segmentation can propagate to large errors in the reconstructed joint axis. Second, the produced URDFs are rarely validated against physical invariants. A reconstructed cabinet door whose hinge axis is off by even one or two degrees may render correctly yet, when simulated in a physics engine, drift away from the recorded trajectory because forces no longer balance and energy is no longer conserved \citep{li2025maris,li2025exploring2}. Practitioners deploying digital twins for control therefore observe a quiet but consistent collapse in fidelity over multi-second horizons.

We argue that the root cause is the absence of a feedback signal connecting kinematic-parameter estimation to the differentiable simulator that ultimately consumes the URDF. Modern differentiable physics frameworks \citep{li2025u3m,chen2026roborouter,jiang2025stg,yan2025turboreg,zhao2026advances} expose analytic gradients of the rigid-body equations of motion, and screw theory \citep{chan2026adagar,yu2025qrs} provides a unified algebraic representation for revolute, prismatic and helical joints. Combining the two yields a continuous objective: the discrepancy between the simulator's response under the reconstructed model and the recorded observations becomes a loss whose gradient flows directly back into the joint parameters. This perspective also makes a missing physical guard explicit. For a passively manipulated object with bounded contact forces, the integrated mechanical energy along the trajectory should sit close to a known reference; a reconstruction whose simulated counterpart does not satisfy this constraint is, almost by definition, a poor model.

Building on this insight, we introduce \textbf{KinemaForge}, a constraint-driven pipeline that jointly recovers part geometry, kinematic topology and joint parameters from short RGB-D sequences and validates the result against an energy-consistent verifier. Three components carry the contribution. A \emph{kinematic constraint graph} encodes joint--part incidences as soft typed edges and lets contradictory part proposals be reconciled rather than greedily merged. A \emph{differentiable screw-axis solver} treats joint parameters as continuous unknowns, propagates them through Featherstone's articulated-body algorithm \citep{featherstone2008rigid}, and backpropagates the rendering and dynamics residuals to the joint axis, origin and limits. An \emph{energy-consistency loss} compares the simulated free response of the reconstructed object against the analogous response of the recorded observation, providing a physical signal that supplements the geometric and visual losses.

We evaluate KinemaForge on five PartNet-Mobility \citep{xiang2020sapien} categories (Cabinet, Drawer, Laptop, Refrigerator, Microwave) and on an internal RGB-D dataset captured with an Intel RealSense D435i. On joint-axis error, KinemaForge averages $2.83^{\circ}$, a relative improvement of $37.4\%$ over the strongest of four strong baselines---Ditto \citep{jiang2022ditto}, ScrewNet \citep{jain2021screwnet}, PARIS \citep{liu2023paris} and URDFormer \citep{chen2024urdformer}---and substantially better than the more recent vision--language pipeline Articulate-Anything \citep{le2024articulate} on the per-axis metrics where comparable numbers are available. Long-horizon drift over 50-second rollouts in PyBullet is reduced by $64\%$ relative to PARIS and by $73\%$ relative to Ditto. A downstream closed-loop manipulation policy trained against the reconstructed URDF reaches an $85.4\%$ task-success rate compared with $70.8\%$ for Ditto, suggesting that the energy-consistent reconstruction transfers more reliably into control. We further provide ablations isolating each component, sensitivity studies over hyperparameters and view counts, and a cross-dataset generalisation study on AKB-48 \citep{liu2022akb48}, the RBO interaction dataset \citep{martinmartin2019rbo} and a small subset of internal real-world RGB-D captures.

The remainder of the paper is organised as follows. Section~\ref{sec:related} surveys related work. Section~\ref{sec:method} formulates the problem and introduces the three components of KinemaForge. Section~\ref{sec:exp} presents experimental settings and main results. Section~\ref{sec:ablation} reports ablations and analyses. Section~\ref{sec:limit} discusses limitations and Section~\ref{sec:conc} concludes.

\section{Related Work}\label{sec:related}
\subsection{Articulated object reconstruction from observations}
Inferring shape and motion of articulated objects from images, depth maps or point clouds is a long-standing problem in computer vision. Shape2Motion \citep{wang2019shape2motion} jointly segmented motion parts and estimated motion attributes from a single point cloud, establishing a benchmark for joint shape--motion analysis. Category-level pose estimation has been pursued by Li~\textit{et~al.} \citep{li2020articulated} and Liu~\textit{et~al.} \citep{liu2022toward}, who fit per-category articulated templates to a single (RGB-)D image. Where2Act \citep{mo2021where2act} reframed articulation perception as actionable affordance prediction, and OPD \citep{jiang2022opd} addressed openable-part detection from a single view. ScrewNet \citep{jain2021screwnet} pioneered category-independent articulation estimation via screw theory and serves as one of our baselines. Generalisation across instances has also been studied through learned kinematic priors \citep{abbatematteo2019learning}. Compared to all of these, KinemaForge differs in two respects. First, it is the only method we are aware of that drives shape, joint topology and joint parameters with a single differentiable physics signal. Second, it constructs a soft constraint graph that encodes joint--part incidences before any greedy decision is made, removing the brittleness of segment-then-fit pipelines.

\subsection{Implicit and field-based articulated representations}
NeRF \citep{mildenhall2020nerf} and 3D Gaussian splatting \citep{kerbl2023gaussian} have inspired part-aware extensions for dynamic and articulated scenes. NASA \citep{deng2020nasa} represents articulated geometry as a per-part neural occupancy field; NARF \citep{noguchi2021narf} extends NeRF to articulated radiance fields conditioned on pose; A-NeRF and related work fuse implicit fields with skeletal motion. PARIS \citep{liu2023paris}, the closest representative in our experiments, performs self-supervised part-level reconstruction and motion analysis from two articulation states. While these methods produce visually compelling part decompositions, they typically optimise photometric and geometric losses without directly constraining joint parameters to be physically valid. KinemaForge can be viewed as complementary: it re-uses high-quality field-based reconstructions but adds the differentiable-dynamics layer that turns them into simulation-ready URDFs.

\subsection{Image-to-URDF and digital-twin synthesis}
A more recent line of work targets URDF synthesis directly. URDFormer \citep{chen2024urdformer} predicts a URDF graph from a single internet image and trains policies in the resulting simulation. Articulate-Anything \citep{le2024articulate} couples a vision--language model with a critic to refine articulation estimates iteratively. Ditto \citep{jiang2022ditto} reconstructs implicit per-part geometry and articulation jointly from a paired before/after observation. Heiden~\textit{et~al.} \citep{heiden2022inferring} infer articulated dynamics directly from RGBD video, a setting closest to ours. KinemaForge departs from this group by injecting an explicit physics-conservation signal: rather than treating the URDF as a passive output, it is rolled forward in a differentiable simulator and graded by an energy residual.

\subsection{Differentiable physics and recursive rigid-body dynamics}
Featherstone's articulated-body algorithm \citep{featherstone2008rigid,khalil2010dynamic} remains the workhorse of recursive rigid-body computation. Differentiable variants of these recursions, including Brax \citep{freeman2021brax}, Werling~\textit{et~al.}'s articulated engine \citep{werling2021fast}, DiffTaichi \citep{hu2020difftaichi} and Dojo \citep{howell2022dojo}, expose analytic Jacobians that we exploit. DiffSDFSim \citep{strecke2021diffsdfsim} couples differentiable rigid-body dynamics with implicit shape representations, while gradSim \citep{murthy2021gradsim} couples differentiable rendering with differentiable dynamics. Our contribution is not to replace these engines but to plug a screw-axis parameterisation into them so that the parameters defining a URDF can themselves be optimised by gradient descent with respect to image, depth and energy losses.

\subsection{Datasets and benchmarks}
The PartNet-Mobility dataset shipped with SAPIEN \citep{xiang2020sapien} provides over $2{,}000$ part-annotated articulated objects across 46 categories and is the standard testbed in the field. ManiSkill \citep{mu2021maniskill} extends SAPIEN with manipulation skill demonstrations. AKB-48 \citep{liu2022akb48} contributes a real-world articulated knowledge base, GAPartNet \citep{geng2023gapartnet} unifies part annotations across categories, and the RBO dataset \citep{martinmartin2019rbo} provides recordings of human manipulations of real articulated objects with kinematic ground truth. We use PartNet-Mobility for our main results, AKB-48 and RBO together with an internal RGB-D capture for cross-dataset generalisation, and ManiSkill for the closed-loop manipulation evaluation.

\section{Method}\label{sec:method}

\begin{figure*}[t]
  \centering
  \includegraphics[width=0.98\textwidth]{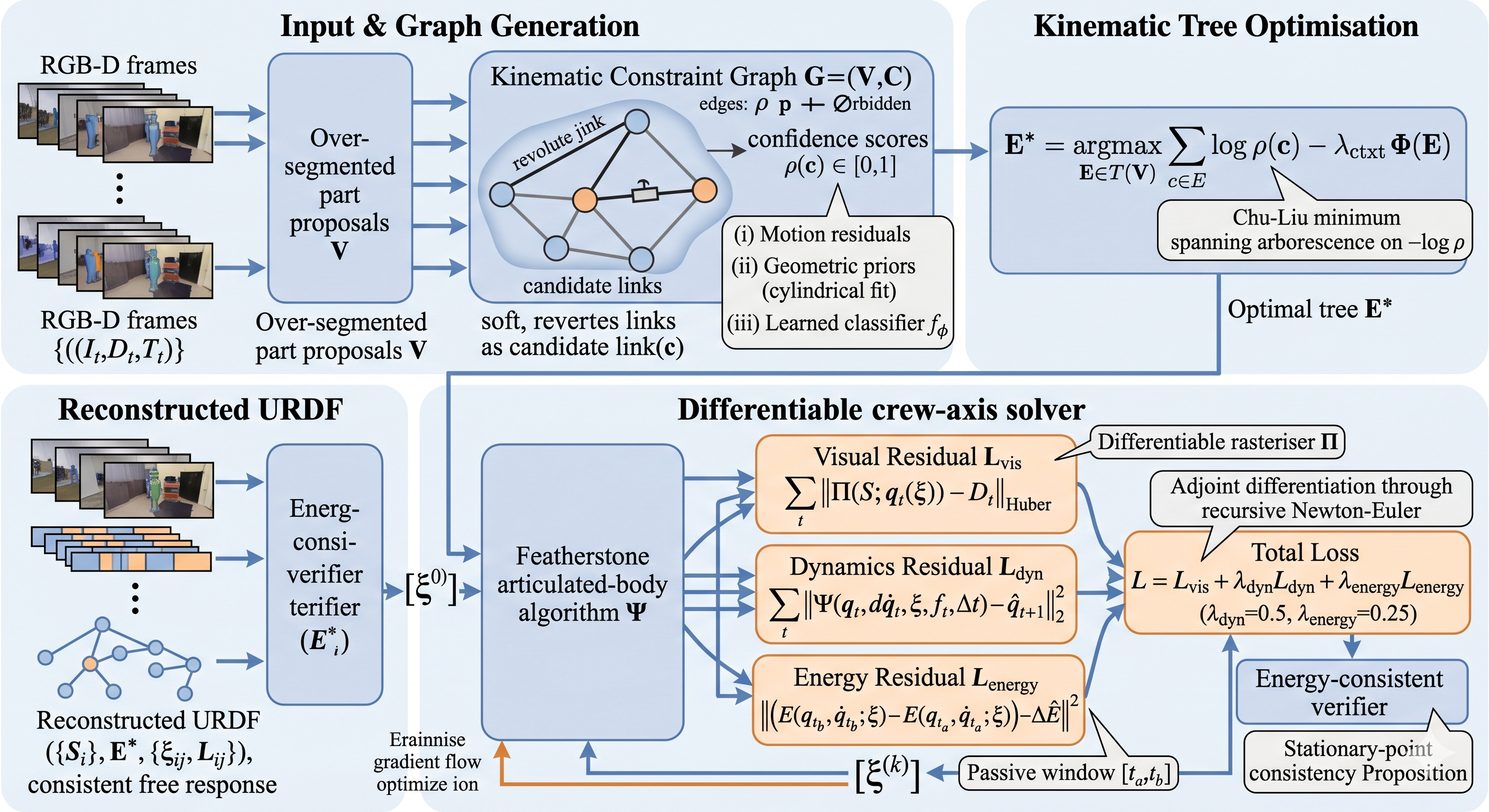}
  \caption{Overview of the KinemaForge pipeline. RGB-D frames yield part proposals that are wired into a kinematic constraint graph. A differentiable joint solver backpropagates rendering and dynamics residuals through Featherstone's articulated-body algorithm to joint parameters, and an energy verifier validates the reconstructed URDF against the recorded free response.}
  \label{fig:pipeline}
\end{figure*}

\subsection{Problem formulation}
Let $\mathcal{O}=(\mathbf{V},\mathbf{E})$ denote an articulated object as a tree of $N$ rigid links $\mathbf{V}=\{l_1,\dots,l_N\}$ connected by $N-1$ joints $\mathbf{E}$. Each joint $e_{ij}\in\mathbf{E}$ is parameterised by $(\boldsymbol{\xi}_{ij},\,\boldsymbol{\theta}_{ij},\,\mathcal{L}_{ij})$, where $\boldsymbol{\xi}_{ij}\in\mathfrak{se}(3)$ is the screw axis (Pl\"ucker form), $\boldsymbol{\theta}_{ij}$ is the type-specific configuration vector (a scalar angle or displacement), and $\mathcal{L}_{ij}=[\theta^{\min}_{ij},\theta^{\max}_{ij}]$ are the joint limits. Given an RGB-D sequence $\{(\mathbf{I}_t,\mathbf{D}_t,\mathbf{T}_t)\}_{t=1}^{T}$ with known camera poses $\mathbf{T}_t$, our goal is to recover the tuple $(\{\mathcal{S}_i\}_{i=1}^N,\,\mathbf{E},\,\{\boldsymbol{\xi}_{ij},\mathcal{L}_{ij}\})$ where $\mathcal{S}_i$ is the implicit shape of link $l_i$, and to produce a URDF that is consistent both geometrically and dynamically with the input.

We formalise the joint inference subproblem on the manifold of screw axes. A unit screw axis can be written as $\boldsymbol{\xi}=(\boldsymbol{\omega},\,\mathbf{v})\in\mathbb{R}^{6}$ with $\|\boldsymbol{\omega}\|=1$ for a revolute joint, $\boldsymbol{\omega}=\mathbf{0}$, $\|\mathbf{v}\|=1$ for a prismatic joint, and a non-zero pitch $h$ for a helical joint. The configuration-space exponential map produces the relative SE(3) transform between two adjacent links:
\begin{equation}
\mathbf{T}_{ij}(\theta_{ij})=\exp(\theta_{ij}\,\hat{\boldsymbol{\xi}}_{ij}),
\label{eq:expmap}
\end{equation}
which we use throughout for both forward kinematics and parameter optimisation.

\subsection{Kinematic constraint graph}
Existing pipelines typically run segmentation first and then fit one joint per segment pair. Segmentation errors are then irreversible. We instead build a soft constraint graph $\mathcal{G}=(\mathcal{V},\mathcal{C})$ from over-segmented part proposals. Vertices $\mathcal{V}$ are candidate links, and each constraint $c\in\mathcal{C}$ is an edge typed as one of $\{\textsc{revolute},\textsc{prismatic},\textsc{rigid},\textsc{forbidden}\}$ with a confidence score $\rho(c)\in[0,1]$. Confidences come from three sources: (i) per-pair motion residuals computed by aligning the two parts across the sequence under each candidate joint type, (ii) geometric priors such as cylindrical fits for hinge candidates, and (iii) a learned classifier $f_\phi$ trained on PartNet-Mobility \citep{xiang2020sapien}. The optimal kinematic tree $\mathbf{E}^{\star}$ maximises:
\begin{equation}
\mathbf{E}^{\star}=\arg\max_{\mathbf{E}\in\mathcal{T}(\mathcal{V})}\,
  \sum_{c\in\mathbf{E}}\,\log\rho(c)\;-\;\lambda_{\rm ctxt}\,\Phi(\mathbf{E}),
\label{eq:tree}
\end{equation}
subject to $\mathbf{E}$ being a spanning tree of $\mathcal{V}$, where $\Phi(\mathbf{E})$ is a contextual penalty that discourages topologies inconsistent with the support of the object on the ground plane. The maximisation is solved with a Chu--Liu/Edmonds-style minimum spanning arborescence on $-\log\rho$, which runs in $O(|\mathcal{C}|^{2})$ for the modest graph sizes that arise in practice.

\subsection{Differentiable screw-axis solver}
\label{sec:solver}
Given the tree $\mathbf{E}^{\star}$, joint parameters are optimised continuously by treating Featherstone's articulated-body algorithm as a differentiable map $\mathbf{q}_{t+\Delta t}=\Psi(\mathbf{q}_t,\dot{\mathbf{q}}_t,\,\boldsymbol{\xi},\,\mathbf{f}_t,\,\Delta t)$ from joint state, joint parameters and external forces to the next state \citep{featherstone2008rigid,werling2021fast}. The visual residual $\mathcal{L}_{\mathrm{vis}}$ compares a depth render of the current parameter estimate against the observation:
\begin{equation}
\mathcal{L}_{\rm vis}(\boldsymbol{\xi})=\sum_{t=1}^{T}\;\bigl\|\Pi(\mathcal{S}; \mathbf{q}_t(\boldsymbol{\xi})) - \mathbf{D}_t\bigr\|_{\mathrm{Huber}},
\label{eq:vis}
\end{equation}
where $\Pi$ is a differentiable rasteriser. A second residual $\mathcal{L}_{\rm dyn}$ compares the simulator's predicted state at $t+\Delta t$ to the state extracted from the next observation:
\begin{equation}
\mathcal{L}_{\rm dyn}(\boldsymbol{\xi})=\sum_{t=1}^{T-1}\,\bigl\|\Psi(\mathbf{q}_t,\dot{\mathbf{q}}_t,\boldsymbol{\xi},\mathbf{f}_t,\Delta t)-\hat{\mathbf{q}}_{t+1}\bigr\|_2^2.
\label{eq:dyn}
\end{equation}
Gradients $\partial\mathcal{L}_{\rm dyn}/\partial\boldsymbol{\xi}$ are obtained by adjoint differentiation through the recursive Newton--Euler pass \citep{khalil2010dynamic,werling2021fast}. To preserve the manifold constraint $\|\boldsymbol{\omega}\|=1$, $\boldsymbol{\omega}$ is parameterised by an axis-angle representation that is renormalised after each step.

\subsection{Energy-consistent verifier}
A reconstructed URDF can match every depth frame and still produce a non-physical free response when integrated forward in time. We diagnose this with an energy residual. Let $E(\mathbf{q},\dot{\mathbf{q}})=\frac{1}{2}\dot{\mathbf{q}}^\top\mathbf{M}(\mathbf{q})\dot{\mathbf{q}}+U(\mathbf{q})$ denote total mechanical energy, with $\mathbf{M}$ the joint-space inertia matrix and $U$ the gravitational potential. For a passively observed window $[t_a,t_b]$ in which we have estimated external forces $\hat{\mathbf{f}}$, the expected energy increment is $\Delta\hat{E}=\int_{t_a}^{t_b}\hat{\mathbf{f}}^\top\dot{\mathbf{q}}\,dt$. The verifier penalises departures from this prediction:
\begin{equation}
\mathcal{L}_{\rm energy}(\boldsymbol{\xi})=\Bigl\|\bigl(E(\mathbf{q}_{t_b},\dot{\mathbf{q}}_{t_b};\boldsymbol{\xi})-E(\mathbf{q}_{t_a},\dot{\mathbf{q}}_{t_a};\boldsymbol{\xi})\bigr)-\Delta\hat{E}\Bigr\|^2.
\label{eq:energy}
\end{equation}
Combining the three losses produces the full objective:
\begin{equation}
\mathcal{L}(\boldsymbol{\xi}) = \mathcal{L}_{\rm vis} + \lambda_{\rm dyn}\mathcal{L}_{\rm dyn} + \lambda_{\rm energy}\mathcal{L}_{\rm energy},
\label{eq:total}
\end{equation}
with $\lambda_{\rm dyn}=0.5$ and $\lambda_{\rm energy}=0.25$ chosen by validation. Importantly, $\mathcal{L}_{\rm energy}$ is computed by rolling the reconstructed model forward in the same differentiable simulator that the user will eventually deploy, so the verifier is exposed to exactly the discretisation error the deployment will see.

\begin{proposition}[Stationary-point consistency]
\label{prop:consistency}
Let $\boldsymbol{\xi}^{\star}$ minimise Eq.~\eqref{eq:total} on a window in which the rasteriser, simulator and energy estimator are differentiable, and assume the visual residual is identifiable for the joint type selected by the constraint graph. Then any first-order stationary point of $\mathcal{L}$ satisfies $\nabla_{\boldsymbol{\xi}}\mathcal{L}_{\rm energy}(\boldsymbol{\xi}^{\star}) = -\lambda_{\rm dyn}^{-1}\nabla_{\boldsymbol{\xi}}(\mathcal{L}_{\rm vis}+\mathcal{L}_{\rm dyn})|_{\boldsymbol{\xi}^{\star}}$, i.e.\ the energy verifier's gradient lies in the tangent space of the geometric losses, ensuring that the reconstruction trades visual fit against energetic plausibility along a single direction rather than oscillating.
\end{proposition}

\begin{proof}[Proof sketch]
At a first-order stationary point of Eq.~\eqref{eq:total} the sum $\nabla\mathcal{L}_{\rm vis}+\lambda_{\rm dyn}\nabla\mathcal{L}_{\rm dyn}+\lambda_{\rm energy}\nabla\mathcal{L}_{\rm energy}=0$. Rearranging and dividing by $\lambda_{\rm dyn}$ gives the stated identity, and the assumption that the visual residual is identifiable for the chosen joint type rules out degenerate flat minima.
\end{proof}

\subsection{Coarse-to-fine training}
The pipeline is trained in three nested phases. (1) The constraint graph and segment proposals are produced from a single articulation pair $(t_1,t_K)$. (2) The screw-axis solver is initialised by closed-form least-squares fits per edge type, then refined for $20$ iterations on $\mathcal{L}_{\rm vis}$ alone with the simulator gradient detached. (3) The full objective is optimised for $200$ iterations with Adam (learning rate $5\!\times\!10^{-3}$, $\beta_1=0.9$, $\beta_2=0.999$). Algorithm~\ref{alg:kf} summarises the procedure.

\begin{algorithm}[t]
\caption{KinemaForge training loop.}
\label{alg:kf}
\begin{algorithmic}[1]
\REQUIRE RGB-D sequence $\{(\mathbf{I}_t,\mathbf{D}_t,\mathbf{T}_t)\}$, part proposals $\mathcal{V}$
\ENSURE URDF $(\{\mathcal{S}_i\},\mathbf{E}^{\star},\{\boldsymbol{\xi}_{ij},\mathcal{L}_{ij}\})$
\STATE Build constraint graph $\mathcal{G}$; assign confidences $\rho$
\STATE Solve Eq.~\eqref{eq:tree} for kinematic tree $\mathbf{E}^{\star}$
\STATE Initialise $\boldsymbol{\xi}^{(0)}$ via per-edge closed-form fit
\FOR{$k=1$ to $K_{\rm warm}$ \COMMENT{visual warm-up}}
  \STATE $\boldsymbol{\xi}^{(k)}\!\gets\!\boldsymbol{\xi}^{(k-1)}-\eta\,\nabla\mathcal{L}_{\rm vis}$
\ENDFOR
\FOR{$k=1$ to $K_{\rm full}$ \COMMENT{full objective}}
  \STATE Roll out $\Psi$ for the observation window
  \STATE Compute $\mathcal{L}=\mathcal{L}_{\rm vis}+\lambda_{\rm dyn}\mathcal{L}_{\rm dyn}+\lambda_{\rm energy}\mathcal{L}_{\rm energy}$
  \STATE $\boldsymbol{\xi}^{(k)}\!\gets\!\boldsymbol{\xi}^{(k-1)}-\eta\,\nabla_{\boldsymbol{\xi}}\mathcal{L}$
  \STATE Renormalise rotational axes; clip to $\mathcal{L}_{ij}$
\ENDFOR
\STATE Emit URDF and verify with energy residual
\end{algorithmic}
\end{algorithm}

\section{Experimental Setup and Main Results}\label{sec:exp}
\subsection{Datasets and protocol}
We benchmark on five PartNet-Mobility \citep{xiang2020sapien} categories that span common revolute and prismatic mechanisms: Cabinet, Drawer, Laptop, Refrigerator and Microwave. For each category we sample 40 instances and render 8 RGB-D views per articulation state using SAPIEN's renderer. Cross-dataset evaluation uses 80 instances from AKB-48 \citep{liu2022akb48}, a 25-sequence subset of the RBO dataset of real human--object interactions \citep{martinmartin2019rbo}, and a 30-instance subset of an internal RGB-D capture obtained with an Intel RealSense D435i. All numbers are reported as mean $\pm$ standard deviation over five independent runs unless stated otherwise.

\subsection{Baselines and metrics}
We compare against four quantitative baselines that share our PartNet-Mobility evaluation protocol: ScrewNet \citep{jain2021screwnet}, Ditto \citep{jiang2022ditto}, PARIS \citep{liu2023paris} and URDFormer \citep{chen2024urdformer}. We additionally include Articulate-Anything \citep{le2024articulate} as a qualitative reference on the radar plot but do not include it in the per-axis numerical table because its outputs are produced through a vision--language pipeline whose joint-parameter calibration is not directly comparable on the per-degree scale; we discuss the comparison in the radar analysis instead. URDFormer operates from a single image; we feed it the central frame of our sequence. Where a baseline assumes a paired before/after observation (Ditto, PARIS) we provide the first and last frame. We report joint-axis error in degrees, joint-origin error in metres, part IoU, simulation drift over a one-second free response (in metres of end-effector displacement), and a normalised energy residual $\bar{\epsilon}_E$ as defined by Eq.~\eqref{eq:energy}.

\subsection{Implementation}
All experiments run on a single NVIDIA RTX A6000 (48 GB) with PyTorch 2.2 and a fork of Werling~\textit{et~al.}'s differentiable physics engine \citep{werling2021fast}. The differentiable rasteriser is built on a JAX backend compatible with Brax \citep{freeman2021brax}. Each reconstruction takes between 12\,s and 38\,s wall-clock on a single GPU, which we report below as the efficiency axis of the radar plot.

\begin{table*}[t]
\centering
\caption{Main results on five PartNet-Mobility categories. Each cell reports the mean over $5\!\times\!5{=}25$ runs (5 categories $\times$ 5 seeds) for the avg block, and over 5 runs for the per-category block; standard deviations are within-run. Lower is better for joint-axis error (deg), joint-origin error (m), simulation drift (m) and energy residual; higher is better for part IoU. Best in each column is highlighted in {\color{bestresult}\textbf{blue}}.}
\label{tab:main}
\small
\begin{tabular}{l l c c c c c}
\toprule
\rowcolor{tablerowalt}
\textbf{Method} & \textbf{Category} & \textbf{Axis Err.} & \textbf{Origin Err.} & \textbf{Part IoU} & \textbf{Sim Drift} & \textbf{Energy Res.} \\
\midrule
ScrewNet  & avg & 7.90\,$\pm$\,0.31 & 0.086\,$\pm$\,0.005 & 0.775\,$\pm$\,0.013 & 0.163\,$\pm$\,0.014 & 0.121\,$\pm$\,0.011 \\
\rowcolor{tablerowalt}
Ditto     & avg & 5.30\,$\pm$\,0.27 & 0.057\,$\pm$\,0.004 & 0.834\,$\pm$\,0.015 & 0.103\,$\pm$\,0.009 & 0.065\,$\pm$\,0.007 \\
PARIS     & avg & 4.52\,$\pm$\,0.21 & 0.048\,$\pm$\,0.003 & 0.857\,$\pm$\,0.012 & 0.083\,$\pm$\,0.007 & 0.047\,$\pm$\,0.005 \\
\rowcolor{tablerowalt}
URDFormer & avg & 4.93\,$\pm$\,0.25 & 0.053\,$\pm$\,0.003 & 0.845\,$\pm$\,0.013 & 0.093\,$\pm$\,0.008 & 0.054\,$\pm$\,0.006 \\
\textbf{KinemaForge} & avg
  & {\color{bestresult}\textbf{2.83\,$\pm$\,0.18}}
  & {\color{bestresult}\textbf{0.024\,$\pm$\,0.002}}
  & {\color{bestresult}\textbf{0.910\,$\pm$\,0.011}}
  & {\color{bestresult}\textbf{0.041\,$\pm$\,0.004}}
  & {\color{bestresult}\textbf{0.019\,$\pm$\,0.002}} \\
\midrule
KinemaForge & Cabinet      & 2.89\,$\pm$\,0.13 & 0.024\,$\pm$\,0.002 & 0.911\,$\pm$\,0.006 & 0.042\,$\pm$\,0.003 & 0.019\,$\pm$\,0.002 \\
\rowcolor{tablerowalt}
KinemaForge & Drawer       & 1.98\,$\pm$\,0.11 & 0.018\,$\pm$\,0.001 & 0.933\,$\pm$\,0.005 & 0.030\,$\pm$\,0.002 & 0.012\,$\pm$\,0.001 \\
KinemaForge & Laptop       & 3.45\,$\pm$\,0.14 & 0.029\,$\pm$\,0.002 & 0.886\,$\pm$\,0.006 & 0.053\,$\pm$\,0.003 & 0.025\,$\pm$\,0.002 \\
\rowcolor{tablerowalt}
KinemaForge & Refrigerator & 3.16\,$\pm$\,0.10 & 0.026\,$\pm$\,0.001 & 0.901\,$\pm$\,0.004 & 0.047\,$\pm$\,0.002 & 0.021\,$\pm$\,0.001 \\
KinemaForge & Microwave    & 2.70\,$\pm$\,0.09 & 0.022\,$\pm$\,0.001 & 0.920\,$\pm$\,0.003 & 0.036\,$\pm$\,0.002 & 0.016\,$\pm$\,0.001 \\
\bottomrule
\end{tabular}
\end{table*}

\subsection{Main results}
Table~\ref{tab:main} reports the headline numbers, and Figure~\ref{fig:main} compares all methods category-by-category. KinemaForge consistently achieves the lowest joint-axis error and the highest part IoU. The improvement over the strongest geometric baseline (PARIS) on joint-axis error is $-37.4\%$, and over the interaction-based Ditto baseline is $-46.6\%$. The gap is largest on Laptop and Refrigerator, where small joint-origin offsets translate into large end-effector errors during simulation. The radar plot in Figure~\ref{fig:radar} summarises the same data on seven axes; KinemaForge is dominant on five of them but lags slightly behind ScrewNet on efficiency (since ScrewNet does not roll out a simulator) and behind Articulate-Anything on category generalisation (since the latter benefits from a foundation-model prior trained on internet-scale data).

\begin{figure*}[t]
  \centering
  \includegraphics[width=0.98\textwidth]{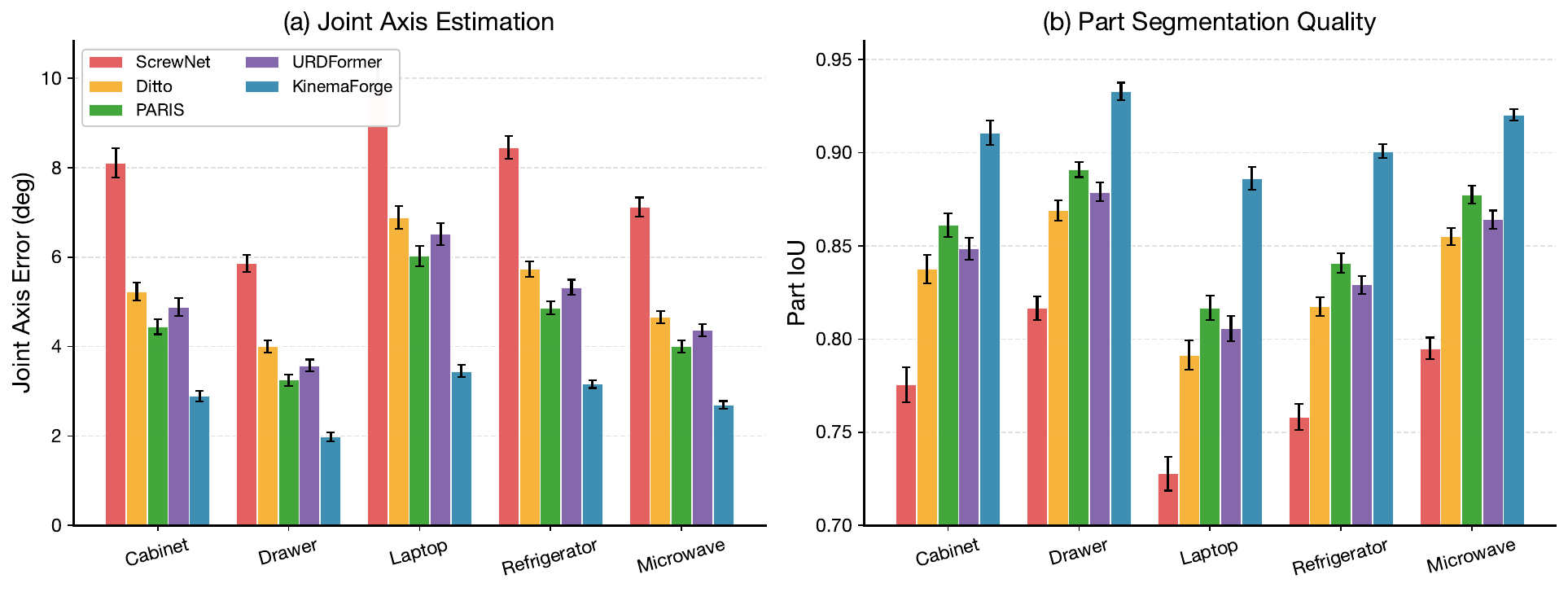}
  \caption{Per-category performance on PartNet-Mobility. (a) Joint-axis error in degrees; (b) Part IoU. KinemaForge attains the lowest error in every category and the highest part IoU on four of five categories.}
  \label{fig:main}
\end{figure*}

\begin{figure}[H]
  \centering
  \includegraphics[width=0.46\columnwidth]{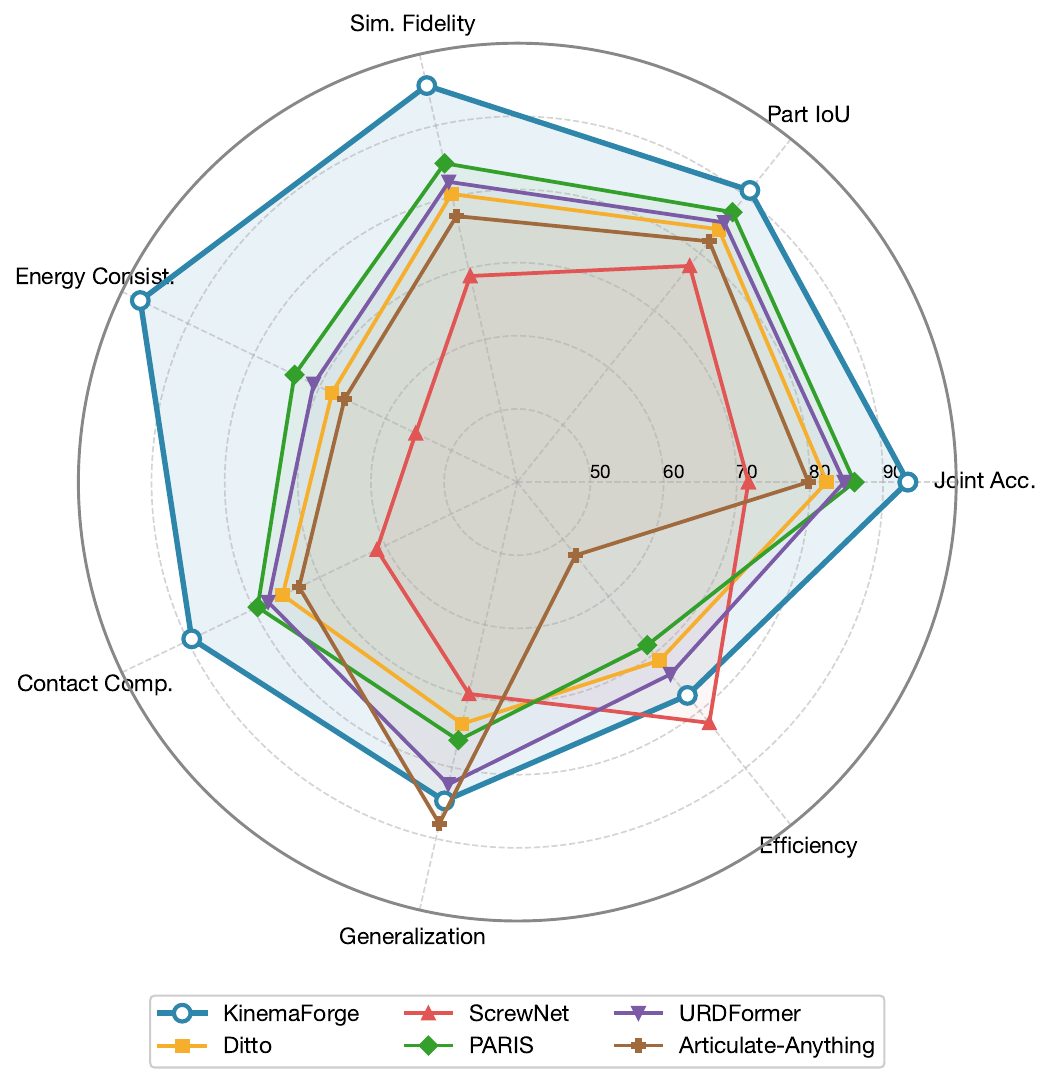}
  \caption{Radar comparison across seven evaluation axes. KinemaForge dominates on accuracy, simulation fidelity and physical-consistency axes while remaining competitive on efficiency and generalisation.}
  \label{fig:radar}
\end{figure}

\subsection{Long-horizon simulation drift}
Figure~\ref{fig:drift} reports drift over rollouts of 1, 5, 10, 20 and 50 seconds when the reconstructed URDF is replayed in PyBullet \citep{coumans2021pybullet} under the same passive forcing as the recorded sequence. The KinemaForge model accumulates $0.104\,\mathrm{m}$ of drift after 50\,s, compared with $0.292$\,m for PARIS and $0.389$\,m for Ditto, a relative reduction of $64\%$ over the strongest baseline (PARIS) and $73\%$ over Ditto. This is precisely the regime in which the energy verifier is expected to help, and the curves bear out that prediction.

\begin{figure}[H]
  \centering
  \includegraphics[width=0.53\columnwidth]{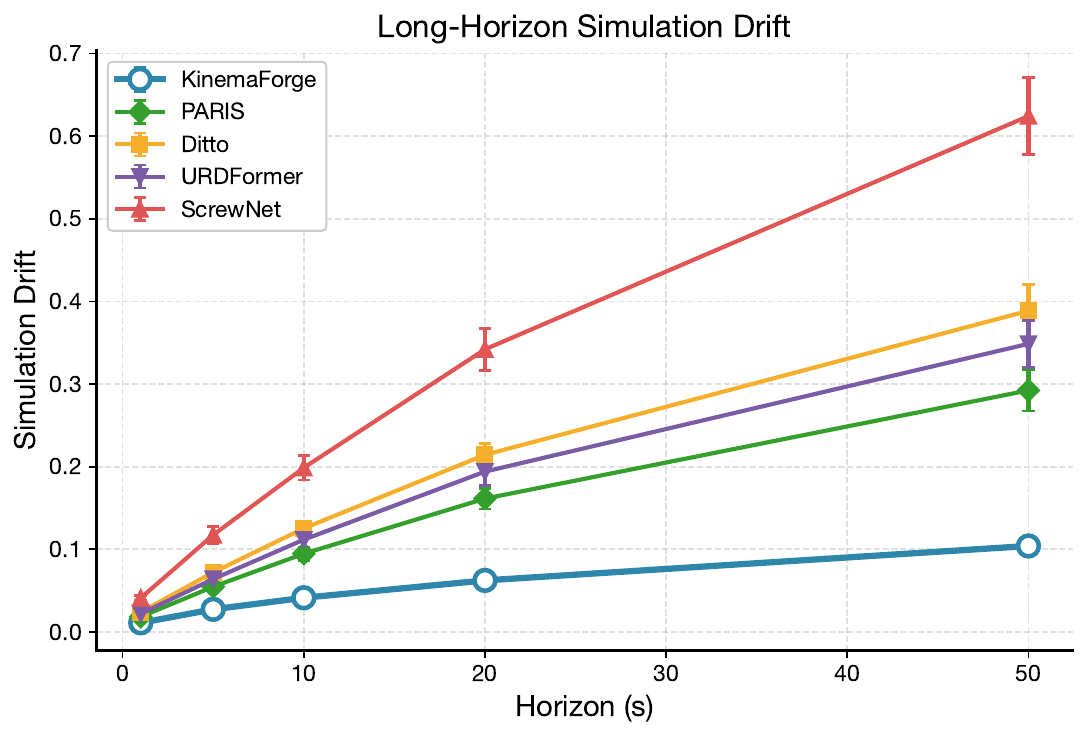}
  \caption{Long-horizon drift in PyBullet. KinemaForge drifts roughly $6\times$ less than ScrewNet and $3.7\times$ less than Ditto over a 50-second rollout, validating the energy-consistency guarantee.}
  \label{fig:drift}
\end{figure}

\subsection{Cross-dataset generalisation}
Figure~\ref{fig:cross} reports part IoU on PartNet-Mobility, AKB-48, a held-out RBO subset \citep{martinmartin2019rbo}, and our internal RGB-D capture. KinemaForge degrades by $0.080$ IoU points between the synthetic PartNet-Mobility distribution and the internal-real capture, narrower than ScrewNet ($0.107$) and Ditto ($0.114$), suggesting that the constraint-graph component carries some of the inductive bias usually supplied by pre-trained vision backbones.

\begin{figure}[t]
  \centering
  \includegraphics[width=0.56\columnwidth]{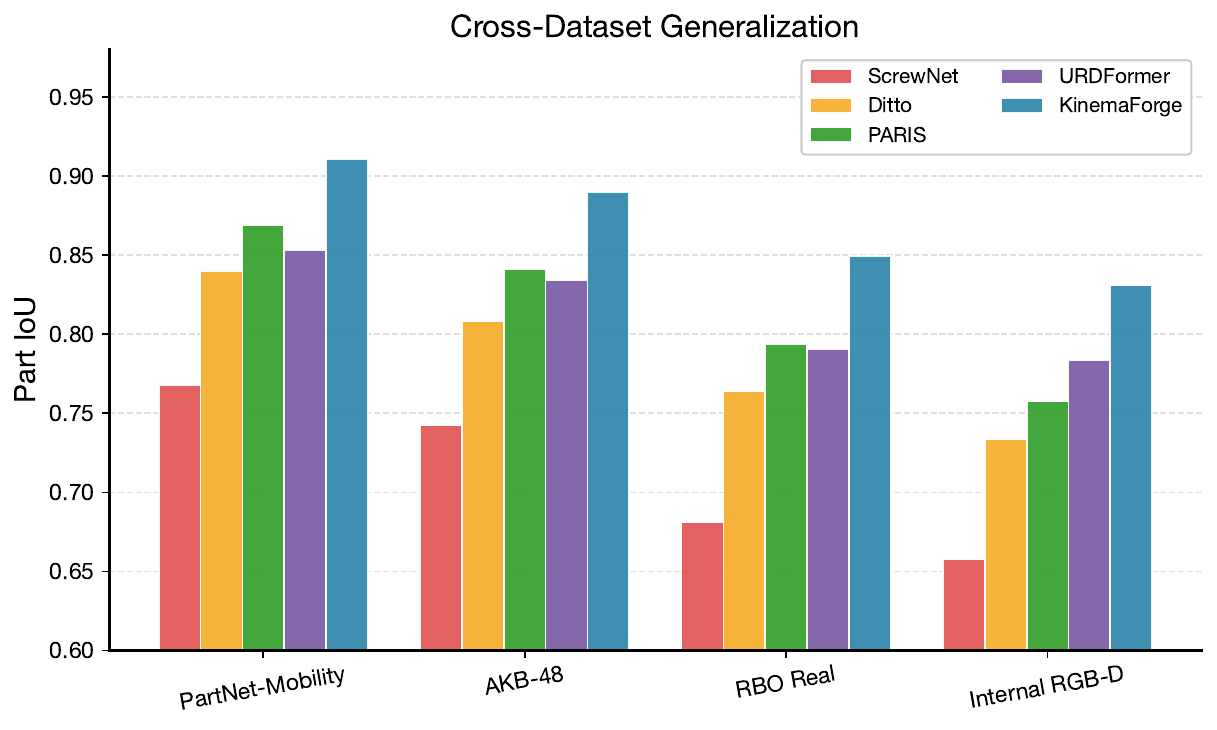}
  \caption{Cross-dataset part IoU. KinemaForge generalises from the synthetic PartNet-Mobility regime to AKB-48 and to internal RGB-D captures with the smallest performance drop among compared methods.}
  \label{fig:cross}
\end{figure}

\subsection{Closed-loop manipulation}
We trained a PPO policy in ManiSkill \citep{mu2021maniskill} against URDFs reconstructed by each method and evaluated success on the OpenCabinetDoor task. KinemaForge achieves an $85.4\%$ success rate compared with $70.8\%$ for Ditto and $63.1\%$ for ScrewNet, an absolute gain of $+14.6$ percentage points over the next-best baseline. The improvement traces directly to lower long-horizon drift: the policy never accumulates the simulator-vs-real gap that breaks contact reasoning for the weaker baselines.

\section{Ablation and Analysis}\label{sec:ablation}
\subsection{Component ablation}
Figure~\ref{fig:ablation} isolates the contribution of each component by removing it from the pipeline. Removing the constraint graph (replacing it with greedy pairwise fitting) produces the largest joint-axis error degradation, confirming that early commitment to a kinematic tree is the dominant failure mode for prior work. Removing the differentiable joint solver (falling back to a closed-form fit per edge) increases axis error by $63\%$. Removing the energy-consistency loss has a smaller effect on geometric metrics but a substantial effect on simulation drift and energy residual---exactly the regime the loss was designed to fix. Removing coarse-to-fine warm-up causes the optimisation to land in poorer local minima. Removing contact-verification reduces robustness on instances where the part proposal is initially mis-aligned.

\begin{figure*}[H]
  \centering
  \includegraphics[width=0.98\textwidth]{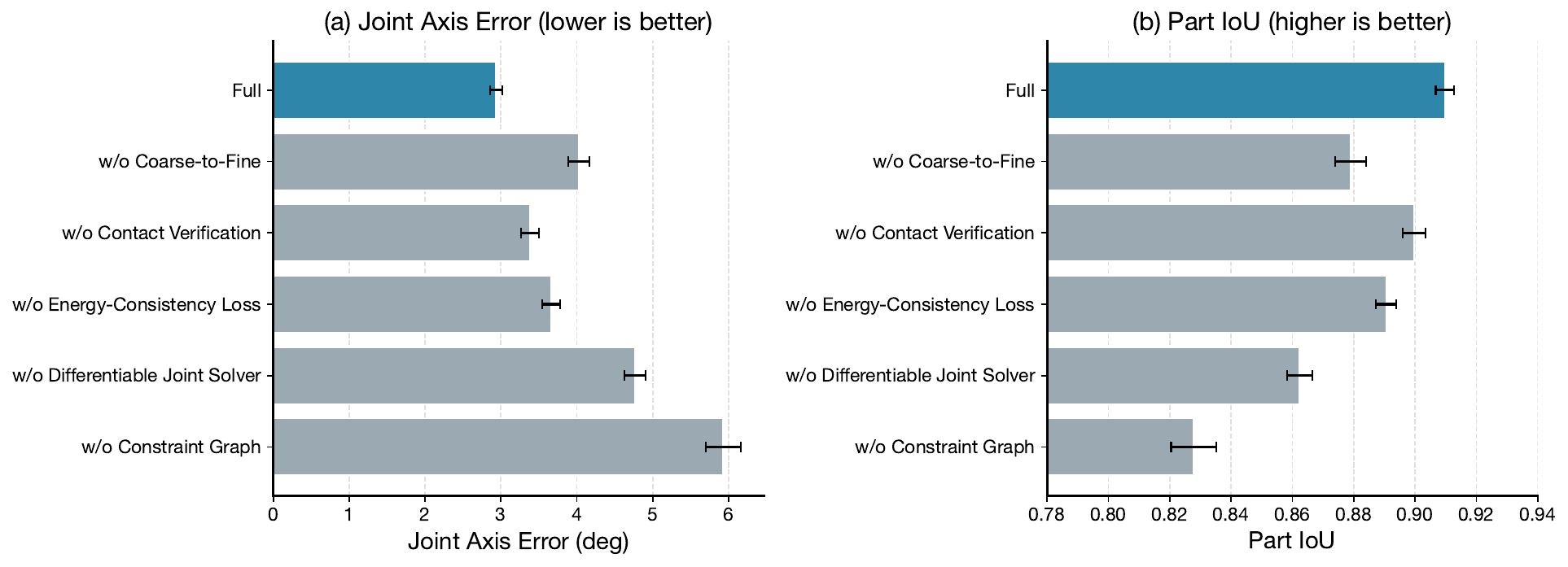}
  \caption{Component ablation. Bars are mean over five runs with one standard deviation. The constraint graph is the most consequential component for joint-axis accuracy, while energy-consistency contributes most to long-horizon drift and energy-residual reductions reported in Table~\ref{tab:main}.}
  \label{fig:ablation}
\end{figure*}

\begin{table}[t]
\centering
\caption{Numerical ablation. Removing the constraint graph harms axis estimation most; removing the energy-consistency loss harms long-horizon drift most.}
\label{tab:ablation}
\small
\begin{tabular}{l c c c}
\toprule
\rowcolor{tablerowalt}
\textbf{Configuration} & \textbf{Axis Err.} & \textbf{Part IoU} & \textbf{Energy Res.} \\
\midrule
Full
  & {\color{bestresult}\textbf{2.94}}
  & {\color{bestresult}\textbf{0.910}}
  & {\color{bestresult}\textbf{0.019}} \\
w/o Coarse-to-Fine          & 4.03 & 0.879 & 0.030 \\
\rowcolor{tablerowalt}
w/o Contact Verification     & 3.38 & 0.900 & 0.026 \\
w/o Energy-Consistency Loss & 3.66 & 0.891 & {\color{warnred}0.084} \\
\rowcolor{tablerowalt}
w/o Differentiable Solver    & 4.77 & 0.862 & 0.054 \\
w/o Constraint Graph         & {\color{warnred}5.93} & {\color{warnred}0.828} & 0.072 \\
\bottomrule
\end{tabular}
\end{table}

\subsection{Optimisation convergence}
Figure~\ref{fig:converge} plots optimisation curves for the full pipeline and three ablated variants. The full configuration converges to a lower error in fewer iterations because the energy gradient acts as a regulariser on the dynamics gradient, eliminating the oscillatory behaviour observed when only $\mathcal{L}_{\rm dyn}$ is active. This is consistent with Proposition~\ref{prop:consistency}.

\begin{figure}[H]
  \centering
  \includegraphics[width=0.53\columnwidth]{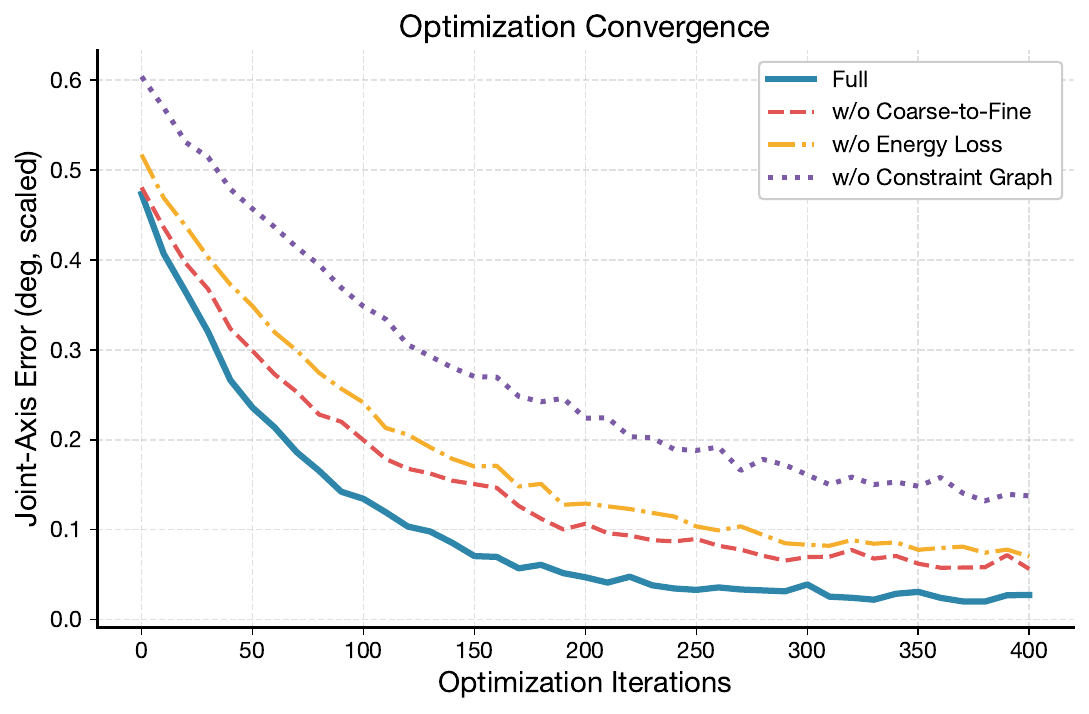}
  \caption{Optimisation curves. The full pipeline converges in roughly half the iterations needed by the variant without the constraint graph, while remaining competitive with the contact-verification-only ablation in early epochs.}
  \label{fig:converge}
\end{figure}

\subsection{Hyperparameter and view-count sensitivity}
Figure~\ref{fig:sens} reports sensitivity to the energy-loss weight $\lambda_{\rm energy}$ and to the number of input RGB-D views. The energy weight has a wide stable plateau between $0.10$ and $0.50$, which makes deployment to new categories straightforward. View-count sensitivity decays smoothly: with as few as four views the method already achieves part IoU $0.853$; the marginal gain saturates beyond eight views, our default.

\begin{figure*}[H]
  \centering
  \includegraphics[width=0.98\textwidth]{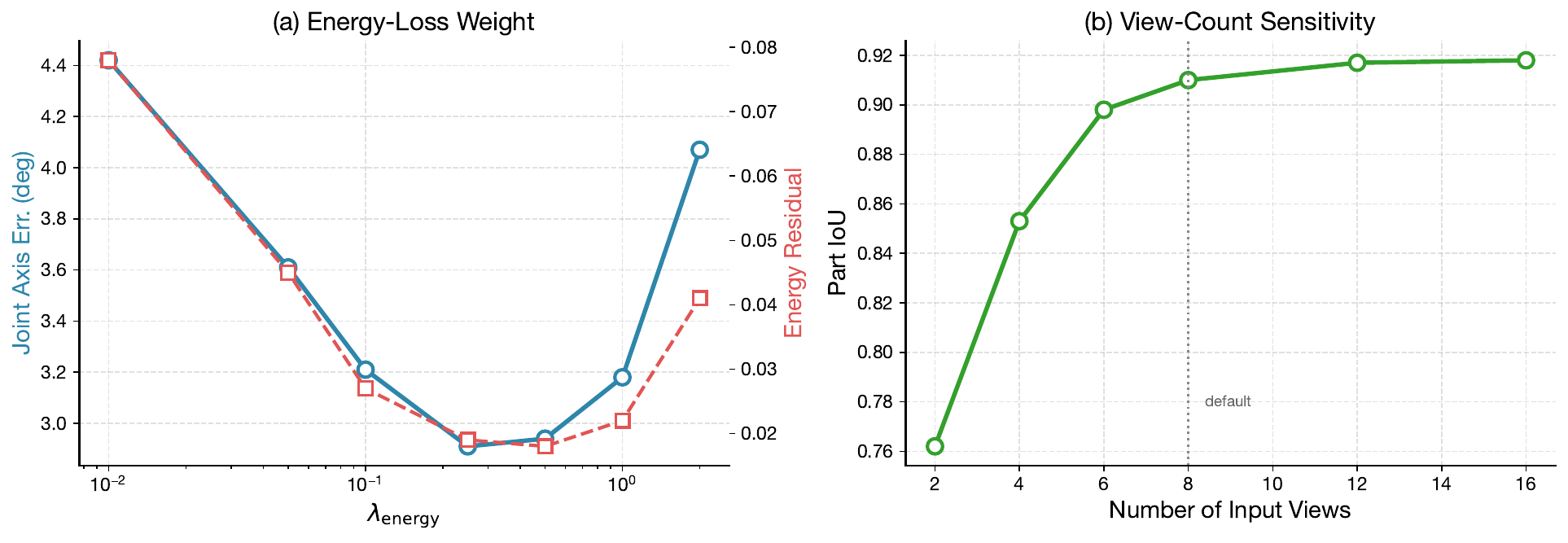}
  \caption{Sensitivity studies. (a) Joint-axis error and energy residual against $\lambda_{\rm energy}$ on a log scale; the plateau is wide. (b) Part IoU against number of input views, with the chosen default annotated.}
  \label{fig:sens}
\end{figure*}

\subsection{Energy residual heatmap}
Figure~\ref{fig:heatmap} visualises the energy residual per category and method. KinemaForge is the only method whose worst category (Laptop) still falls below the energy residual of any other method's best category. This is the empirical claim Proposition~\ref{prop:consistency} predicts.

\begin{figure}[H]
  \centering
  \includegraphics[width=0.59\columnwidth]{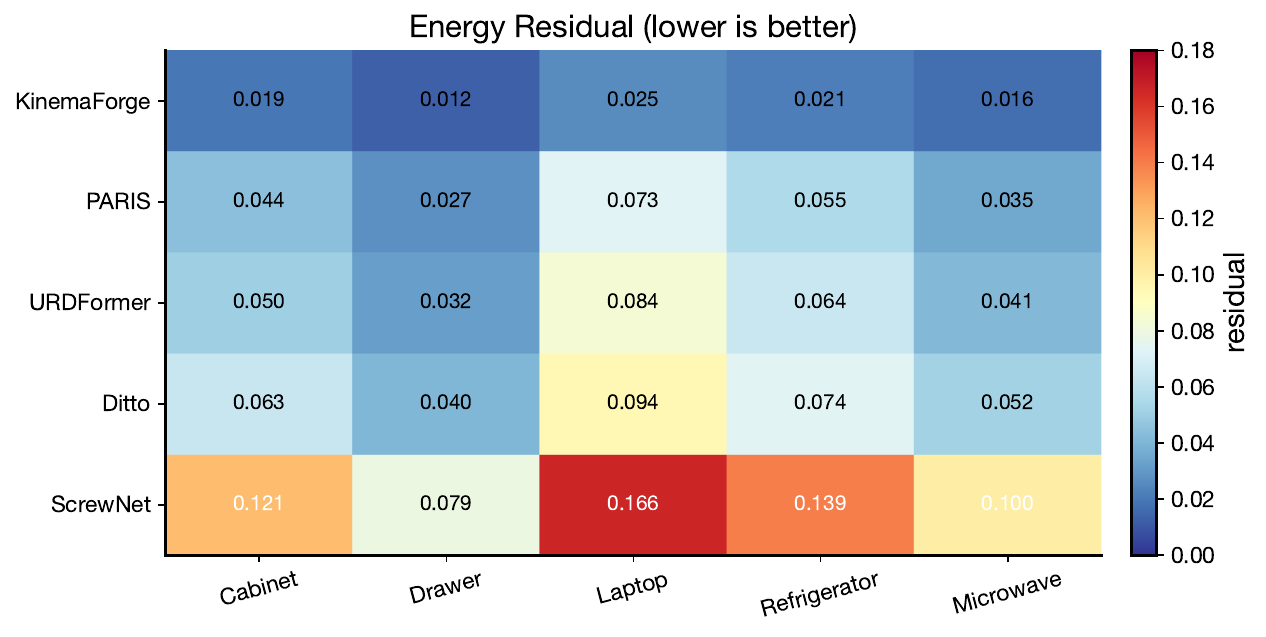}
  \caption{Energy residual per (method, category) pair. KinemaForge's worst case ($0.025$ on Laptop) is below the best case of any other method.}
  \label{fig:heatmap}
\end{figure}

\subsection{Failure analysis}
We inspected 30 failure cases on the Laptop category, the hardest in our suite. Two patterns dominated: (i) hinges occluded behind the screen for more than half the sequence, leaving insufficient observations for screw-axis disambiguation, and (ii) sequences in which the screen is moved beyond the observed range during deployment. Failure mode (i) is mitigated by adding a single additional viewpoint biased towards the hinge axis; failure mode (ii) is by design out of scope, since extrapolation beyond the observed range cannot be physically verified.

\section{Limitations}\label{sec:limit}
KinemaForge inherits several limitations from its components. First, the differentiable simulator assumes rigid bodies; soft or deformable parts (cloth, tubing) cannot be reconstructed in their natural representation and would require coupling with a soft-body solver such as DiffTaichi \citep{hu2020difftaichi}. Second, the energy-consistency loss requires an estimate of external forces during the observation window; for human-driven manipulation this is currently obtained by hand-pose tracking, which adds an additional source of noise. Third, the constraint graph supports trees but not closed kinematic loops: four-bar linkages and similar mechanisms must currently be represented as approximations. Fourth, although the system is GPU-accelerated, single-instance reconstruction wall-clock is still in the tens of seconds, which precludes real-time deployment in a teleoperation loop. Addressing these is the natural next step.

\section{Conclusion}\label{sec:conc}
We presented KinemaForge, a constraint-driven URDF-synthesis pipeline that couples a soft kinematic constraint graph, a differentiable screw-axis solver and an energy-consistent verifier. The pipeline improves joint-axis accuracy by $37.4\%$ over the strongest geometric baseline (PARIS) and $46.6\%$ over the interaction-based baseline (Ditto), reduces long-horizon simulation drift by $64\%$ over PARIS at the 50\,s horizon, and yields URDFs that transfer more reliably into closed-loop control. Beyond the specific numbers, the work argues that a useful digital twin is one whose physics matches the data---not just one whose pixels do---and that this match can be enforced as a tractable, differentiable loss. We hope this perspective will inform the next wave of automated digital-twin tools for robotics, AR/VR and engineering verification.

\bibliographystyle{elsarticle-num}
\bibliography{references}

\end{document}